\documentclass[letterpaper, 10 pt, conference]{ieeeconf}

\IEEEoverridecommandlockouts
\usepackage{amsmath,amssymb}
\usepackage{graphicx}
\usepackage{booktabs}
\usepackage{cite}
\usepackage{url}
\usepackage{algorithm}
\usepackage{algorithmic}
\usepackage{balance}

\newtheorem{theorem}{Theorem}
\newtheorem{definition}{Definition}

\title{\LARGE \bf
Self-Correction as Feedback Control: Error Dynamics, Stability Thresholds, and Prompt Interventions in LLMs
}

\author{Aofan Liu$^{1}$ \quad Jingxiang Meng$^{2,*}$%
\thanks{$^{1}$Peking University.}%
\thanks{$^{2}$University of Chicago.}%
\thanks{$^{*}$Corresponding author.}%
}

\begin{document}

\maketitle
\thispagestyle{empty}
\pagestyle{empty}

%%%%%%%%%%%%%%%%%%%%%%%%%%%%%%%%%%%%%%%%%%%%%%%%%%%%%%%%%%%%%%%%%%%%%%%%%%%%%%%%
\begin{abstract}
Iterative self-correction is increasingly deployed in agentic LLM systems, yet whether repeated refinement improves or degrades performance remains inconsistent across models. We recast self-correction as a closed-loop feedback-control problem in which the same model is both controller and plant, and analyze its error dynamics via a two-state Markov model over \{Correct, Incorrect\}, parameterized by the Error Introduction Rate (EIR) and Error Correction Rate (ECR). The model yields a directly measurable stability threshold---iterate only when $\mathrm{ECR}/\mathrm{EIR} > \mathrm{Acc}/(1-\mathrm{Acc})$---in which EIR acts as a stability margin and prompting becomes lightweight controller design. Empirically, across 7 models and 3 datasets (GSM8K, MATH, StrategyQA), a sharp near-zero EIR boundary ($\lesssim 0.5\%$) cleanly separates beneficial from harmful self-correction: only o3-mini ($+3.4$~pp), Claude Opus~4.6 ($+0.6$~pp), and o4-mini ($\pm 0$~pp) stay non-degrading, while GPT-5 and four others lose accuracy. A \emph{verify-first} prompt intervention then provides causal evidence: it drives GPT-4o-mini's EIR from 2\% to 0\% and converts a $-6.2$~pp degradation into $+0.2$~pp (paired McNemar, $p<10^{-4}$), with negligible change on already-sub-threshold models---exactly as the diagnostic predicts. A complementary analysis of adaptive self-consistency (ASC) shows it halts harmful refinement at a $3.8$~pp confidence-elicitation cost, exposing a two-tier capability structure: prompt-level EIR suppression prevents degradation, whereas ECR enhancement---plausibly training-level---is required for genuine gains. Self-correction should thus be treated not as a default behavior but as a control decision governed by measurable error dynamics.
\end{abstract}

%%%%%%%%%%%%%%%%%%%%%%%%%%%%%%%%%%%%%%%%%%%%%%%%%%%%%%%%%%%%%%%%%%%%%%%%%%%%%%%%
\section{INTRODUCTION}

Iterative self-correction---where a language model reviews and revises its own outputs---has emerged as a central mechanism in agentic AI systems. The promise is intuitive: by reflecting on potential errors, models can catch and fix mistakes that arise in initial reasoning. This capability underpins self-refinement loops in autonomous agents~\cite{madaan2023selfrefine,shinn2023reflexion}, multi-agent debate systems~\cite{chen2024magicore}, and reasoning chains with verification steps~\cite{du2022energy}.

However, a growing body of evidence questions the effectiveness of unbounded self-correction. Huang et al.~\cite{huang2024selfcorrect} demonstrate that LLMs cannot reliably self-correct reasoning without external feedback, often degrading performance through overcorrection. A complementary pattern---which we term the \emph{accuracy--correction paradox}---is that models with higher initial accuracy tend to benefit less from, or even be harmed by, self-correction; related critical surveys~\cite{kamoi2024can,stechly2025self} document similar trends. Chen et al.~\cite{chen2024magicore} find that excessive refinement paradoxically degrades performance in multi-agent systems. These findings converge on a critical question: \emph{What governs the transition from beneficial to harmful self-correction, and can we predict and exploit this transition?}

Prior work has approached this question empirically, characterizing saturation points for specific models and tasks. However, a \emph{theoretical} understanding of why self-correction converges---or diverges---and when it should stop is lacking.

\textbf{This work.} We treat iterative self-correction as a closed-loop control problem: each revision applies a feedback action to the previous answer, and the central question is whether the loop is stabilizing or destabilizing. Our contribution is less a new Markov theorem than an \emph{operationalization} of this feedback view into a measurable diagnostic and an actionable intervention. Specifically, we make three contributions:

\textbf{(1) A Markov-chain diagnostic for self-correction.} We model iterative self-correction as a two-state Markov chain on \{Correct, Incorrect\}, parameterized by the Error Introduction Rate (EIR) and Error Correction Rate (ECR). The equilibrium, steady-state, and convergence expressions are standard consequences of this model; our contribution is to turn them into a directly measurable \emph{stop-or-iterate} diagnostic for deployment. Relative to Yang et al.~\cite{yang2025probabilistic}, who characterize \emph{convergence curves} across rounds, we emphasize an engineering question: when should a controller continue applying self-correction, and when should it stop?

\textbf{(2) A near-zero EIR threshold, causally validated.} Across 7 models and 3 datasets we identify a sharp threshold: only models with EIR $\lesssim$ 0.5\% (o3-mini, Claude Opus~4.6, o4-mini) benefit from self-correction; five others degrade, including GPT-5. A verify-first prompt ablation drives GPT-4o-mini's EIR from 2\% to 0\% and turns $-6.2$~pp degradation into $+0.2$~pp (paired McNemar $p < 10^{-4}$). The same prompt produces little change on already-sub-threshold models, as predicted by the diagnostic, suggesting that EIR is the causal control variable rather than a spurious correlate.

\textbf{(3) Analysis of stopping and capability trade-offs.} We analyze ASC as an adaptive stop-rule combining instance-level self-confidence with batch-level $\widehat{\text{EIR}}/\widehat{\text{ECR}}$ monitoring. ASC halts harmful refinement for GPT-4o-mini, but the confidence-elicitation prompt itself costs 3.8~pp of accuracy. We therefore treat ASC not as a headline gain, but as evidence for a two-tier capability view: \emph{EIR suppression} (prompt-level, already achievable) prevents degradation, whereas \emph{ECR enhancement} (likely training-level) is required for actual improvement.

Our framework shifts the study of self-correction from observation to principled control: measure EIR on a calibration set, check whether the equilibrium condition is satisfied, and iterate only if it is. For high-accuracy models, $\pi^*$ can be far below baseline, making unconditional self-correction a systematic form of compute waste.

%%%%%%%%%%%%%%%%%%%%%%%%%%%%%%%%%%%%%%%%%%%%%%%%%%%%%%%%%%%%%%%%%%%%%%%%%%%%%%%%
\section{RELATED WORK}

\subsection{Self-Correction in Large Language Models}

Huang et al.~\cite{huang2024selfcorrect} demonstrate that LLMs cannot self-correct reasoning without external feedback, finding that self-correction often degrades performance when models lack ground-truth signals. Kamoi et al.~\cite{kamoi2024can} and Stechly et al.~\cite{stechly2025self} provide critical surveys showing that strong self-verification is rare in current LLMs, and that higher baseline accuracy can correlate with worse net gains from self-correction---a pattern we refer to as the \emph{accuracy--correction paradox}.

Yang et al.~\cite{yang2025probabilistic} recently develop a probabilistic theory for multi-round self-correction, modeling accuracy evolution as a Markov process and deriving closed-form convergence curves that align well with empirical trajectories across diverse models. Building on this line of analytical work, we focus on the \emph{practical} question of when self-correction helps versus hurts: we identify a near-zero EIR threshold as a diagnostic for beneficial self-correction and validate it causally through a verify-first prompt ablation.

Our work extends these analyses by examining error propagation across \emph{multiple} iterations rather than single-step correction.

\subsection{Iterative Refinement Strategies}

Madaan et al.~\cite{madaan2023selfrefine} introduce Self-Refine, demonstrating improvements across diverse tasks but without systematically quantifying diminishing returns. Shinn et al.~\cite{shinn2023reflexion} propose Reflexion with verbal reinforcement learning through self-reflection, but rely on external feedback. Chen et al.~\cite{chen2024magicore} introduce MAgICoRe, identifying ``excessive refinement'' as a key challenge but without measuring iteration-by-iteration trajectories. Our work complements these approaches by providing fine-grained empirical measurements of what happens at each refinement step.

\subsection{Agentic Reasoning Systems}

Agentic reasoning systems combine planning, tool use, and self-reflection to tackle open-ended problems through continual interaction with an environment. Du et al.~\cite{du2022energy} explore iterative reasoning through energy minimization, showing that adaptive computational budgets improve performance, while Schick et al.~\cite{schick2023toolformer} show that language models can learn to invoke external tools within their reasoning trajectories. Gou et al.~\cite{gou2024critic} further demonstrate that tool-interactive critiquing is substantially more reliable than intrinsic self-critique. Our focus on refinement depth complements this body of work by quantifying when iteration benefits diminish even in the purely intrinsic setting.

\subsection{Error Analysis in Mathematical Reasoning}

Prior work on mathematical reasoning has focused on improving baseline accuracy through better prompting~\cite{wei2022chain}, fine-tuning~\cite{cobbe2021gsm8k}, or tool augmentation~\cite{schick2023toolformer}. However, systematic analysis of error types and their evolution through refinement cycles remains limited. Existing error taxonomies~\cite{hendrycks2021measuring} classify errors at a single point rather than tracking how errors transform through iterative refinement.

Our work addresses this gap by introducing a dynamic error classification framework that tracks error type transitions across iterations. By analyzing EIR and ECR separately, we provide fine-grained insights into which types of problems benefit from refinement and which are resistant to correction. Despite extensive work on self-correction~\cite{huang2024selfcorrect,madaan2023selfrefine,shinn2023reflexion}, no prior study systematically measures iteration-by-iteration accuracy trajectories, error introduction versus correction rates, and error propagation patterns. Our work fills this gap.

%%%%%%%%%%%%%%%%%%%%%%%%%%%%%%%%%%%%%%%%%%%%%%%%%%%%%%%%%%%%%%%%%%%%%%%%%%%%%%%%
\section{METHOD}
\label{sec:method}

\subsection{Problem Formulation}

Consider $N$ reasoning problems $\{q_1, \ldots, q_N\}$ with ground-truth answers $\{a_1^*, \ldots, a_N^*\}$. Iterative self-correction proceeds as:
\begin{align}
    r_i^{(0)} &= \mathcal{M}(q_i) \label{eq:baseline} \\
    r_i^{(k)} &= \mathcal{M}(q_i, r_i^{(k-1)}, p_{\text{refine}}) \quad \text{for } k \geq 1 \label{eq:refine}
\end{align}
where $r_i^{(k)}$ is the response at iteration $k$ and $p_{\text{refine}}$ is the refinement prompt. Define $c_i^{(k)} = \mathbb{1}[\text{extract}(r_i^{(k)}) = a_i^*]$ and $\text{Acc}(k) = \frac{1}{N}\sum_{i=1}^N c_i^{(k)}$.

\begin{definition}[Error Transition Rates]
At iteration $k$, the Error Introduction Rate (EIR) and Error Correction Rate (ECR) are:
\begin{align}
    \text{EIR}(k) &= P(c_i^{(k+1)} = 0 \mid c_i^{(k)} = 1) \\
    \text{ECR}(k) &= P(c_i^{(k+1)} = 1 \mid c_i^{(k)} = 0)
\end{align}
The Net Benefit is $\text{NB}(k) = \text{Acc}(k) - \text{Acc}(k-1)$.
\end{definition}

\subsection{Markov Chain Model}

\begin{figure*}[t]
\centering
\includegraphics[width=\textwidth]{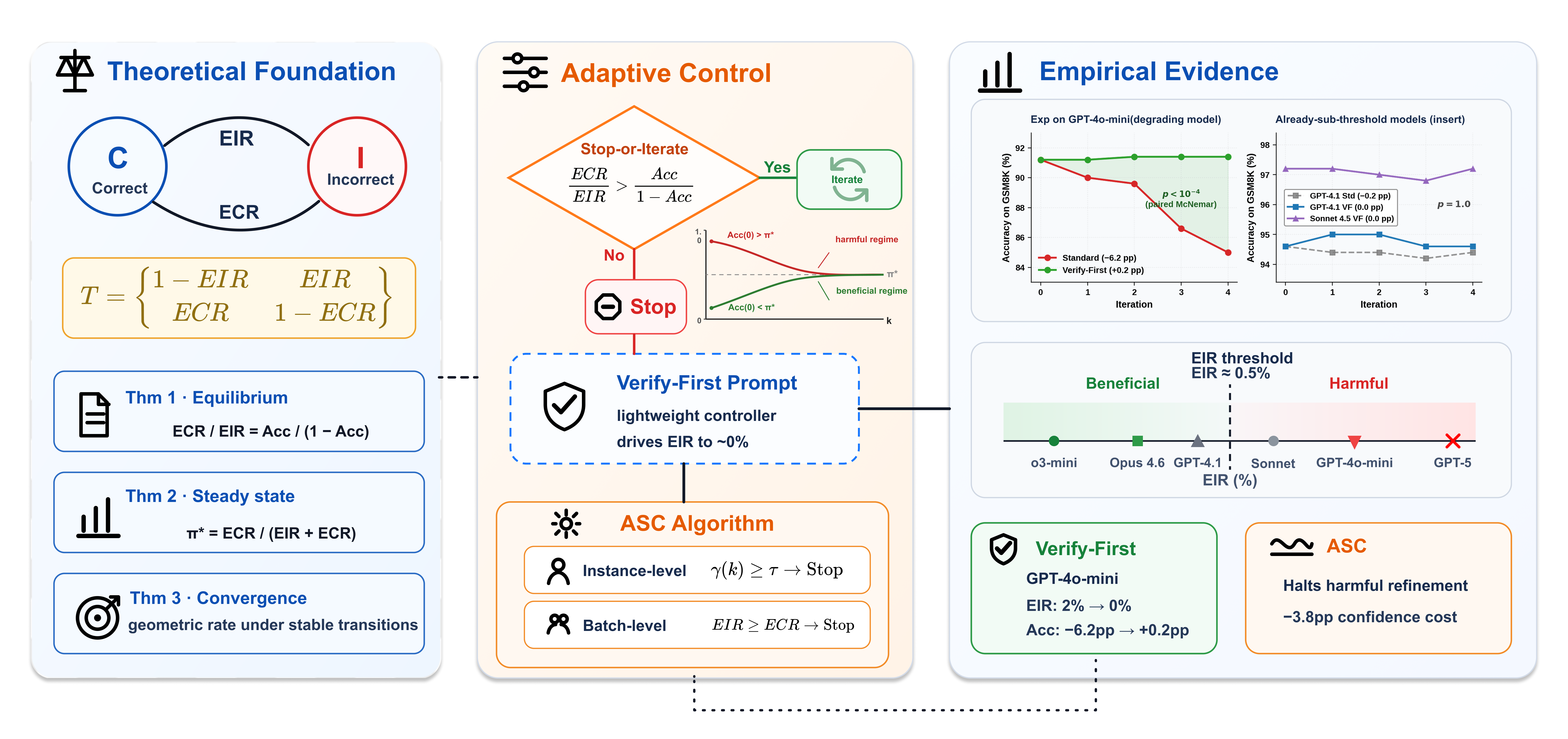}
\caption{Three-layer view of iterative self-correction as a Markov feedback loop. The \emph{theoretical layer} formalises correctness evolution on $\{C,I\}$ with EIR/ECR transitions, yielding equilibrium, steady-state, and convergence expressions. The \emph{control layer} interprets EIR as a stability margin and verify-first prompting as controller design; ASC adds instance-level confidence $\gamma(k)\!\geq\!\tau$ with batch-level $\widehat{\text{EIR}}/\widehat{\text{ECR}}$ monitoring for early stopping. The \emph{empirical layer} evaluates 7 models $\times$ 3 datasets, confirming near-zero EIR ($\lesssim 0.5\%$) as the threshold separating beneficial from harmful self-correction.}
\label{fig:markov}
\end{figure*}

We model each problem's correctness as a two-state Markov chain on $\mathcal{S} = \{\textsc{Correct}, \textsc{Incorrect}\}$, illustrated in Figure~\ref{fig:markov}, with transition matrix:
\begin{equation}
T(k) = \begin{pmatrix} 1 - \text{EIR}(k) & \text{EIR}(k) \\ \text{ECR}(k) & 1 - \text{ECR}(k) \end{pmatrix}
\label{eq:transition}
\end{equation}

\subsection{Theoretical Results}

\begin{theorem}[Equilibrium Condition]
\label{thm:equilibrium}
The net benefit is zero ($\text{NB}(k+1) = 0$) if and only if:
\begin{equation}
\frac{\text{ECR}(k)}{\text{EIR}(k)} = \frac{\text{Acc}(k)}{1 - \text{Acc}(k)}
\label{eq:equilibrium}
\end{equation}
\end{theorem}

This provides an interpretable condition: refinement stops improving when the correction rate (weighted by error proportion) exactly balances the introduction rate (weighted by correct proportion).

\begin{theorem}[Steady-State Accuracy]
\label{thm:steadystate}
If $\text{EIR}(k) \to \text{EIR}^*$ and $\text{ECR}(k) \to \text{ECR}^*$, the steady-state accuracy is:
\begin{equation}
\pi^* = \frac{\text{ECR}^*}{\text{EIR}^* + \text{ECR}^*}
\label{eq:steadystate}
\end{equation}
\end{theorem}

\begin{theorem}[Convergence Rate]
\label{thm:convergence}
Under stationary rates, convergence is geometric:
\begin{equation}
|\text{Acc}(k) - \pi^*| = |1 - \text{EIR}^* - \text{ECR}^*|^k \cdot |\text{Acc}(0) - \pi^*|
\end{equation}
\end{theorem}

\textit{Proof sketch for Theorem~\ref{thm:equilibrium}.} From Markov dynamics, $\text{Acc}(k{+}1) = \text{Acc}(k)(1 - \text{EIR}(k)) + (1 - \text{Acc}(k))\text{ECR}(k)$. Setting $\text{NB} = \text{Acc}(k{+}1) - \text{Acc}(k) = 0$ and rearranging yields Eq.~\eqref{eq:equilibrium}. Theorem~\ref{thm:steadystate} follows from the stationary distribution of the limiting transition matrix, and Theorem~\ref{thm:convergence} from the spectral decomposition with subdominant eigenvalue $\lambda_2 = 1 - \text{EIR}^* - \text{ECR}^*$.

\textbf{Implications.} Theorem~\ref{thm:convergence} shows that when both rates are low (typical for strong models on moderate tasks), convergence is slow ($\rho \approx 1$). When rates are high, convergence is fast but the steady-state may be poor. Most critically, these results reveal a fundamental asymmetry: for high-accuracy models, $\pi^*$ can be far \emph{below} baseline accuracy, making all self-correction harmful.

\subsection{Adaptive Self-Correction (ASC)}

Motivated by the theoretical analysis, we propose ASC (Algorithm~\ref{alg:asc}), which uses two complementary stopping criteria:

\begin{algorithm}[t]
\caption{Adaptive Self-Correction (ASC)}
\label{alg:asc}
\begin{algorithmic}[1]
\REQUIRE Problem $q$, model $\mathcal{M}$, max iterations $K$, threshold $\tau$
\ENSURE Final response $r$, iterations used $k$
\STATE $r^{(0)} \leftarrow \mathcal{M}(q)$; $\gamma^{(0)} \leftarrow \text{ConfScore}(\mathcal{M}, q, r^{(0)})$
\IF{$\gamma^{(0)} \geq \tau$}
    \RETURN $r^{(0)}$, $0$
\ENDIF
\FOR{$k = 1$ to $K$}
    \STATE $r^{(k)} \leftarrow \mathcal{M}(q, r^{(k-1)}, p_{\text{refine}})$
    \STATE $\gamma^{(k)} \leftarrow \text{ConfScore}(\mathcal{M}, q, r^{(k)})$
    \IF{$\gamma^{(k)} \geq \tau$}
        \RETURN $r^{(k)}$, $k$ \hfill\textit{// Confidence criterion}
    \ENDIF
    \IF{$\widehat{\text{EIR}} \geq \widehat{\text{ECR}}$}
        \RETURN $r^{(k-1)}$, $k-1$ \hfill\textit{// Equilibrium criterion}
    \ENDIF
\ENDFOR
\RETURN $r^{(K)}$, $K$
\end{algorithmic}
\end{algorithm}

ASC operates at two levels: (1) \textbf{instance-level confidence}: if self-assessed confidence $\gamma^{(k)} \geq \tau$ (default $\tau = 8$ on 1--10 scale), refinement stops; (2) \textbf{batch-level equilibrium}: if running $\widehat{\text{EIR}} \geq \widehat{\text{ECR}}$, further iterations would decrease accuracy by Theorem~\ref{thm:equilibrium}.

\subsection{Experimental Design}

\textbf{Models.} Seven models spanning four capability tiers: GPT-4o-mini (fast), GPT-4.1 and Claude Sonnet~4 (mid), GPT-5 and Claude Opus~4.6 (frontier), and o3-mini plus o4-mini (reasoning/RLVR). Detailed per-iteration trajectories are reported for six core models on GSM8K (Table~\ref{tab:gsm8k_iter}); o4-mini is evaluated as an additional RLVR validation run in Section~\ref{sec:rlvr}.

\textbf{Datasets.} GSM8K~\cite{cobbe2021gsm8k} (500 problems), MATH~\cite{hendrycks2021math} (400 problems), StrategyQA~\cite{geva2021strategyqa} (200 problems).

\textbf{Baselines.} Generic refinement (4 iterations), Self-Refine~\cite{madaan2023selfrefine}, Reflexion~\cite{shinn2023reflexion}, Self-Consistency~\cite{wang2023selfconsistency}.

%%%%%%%%%%%%%%%%%%%%%%%%%%%%%%%%%%%%%%%%%%%%%%%%%%%%%%%%%%%%%%%%%%%%%%%%%%%%%%%%
\section{EXPERIMENTS}
\label{sec:experiments}

\subsection{Baseline Accuracy}

Table~\ref{tab:baseline} reports baseline accuracies. Opus~4.6 leads on all three benchmarks; MATH is consistently 10--14~pp lower than GSM8K. The one outlier is o3-mini on StrategyQA ($47.0\%$), where the model frequently returns meta-commentary that our exact-match extractor scores as incorrect; this does not affect the GSM8K EIR/ECR analysis below.

\begin{table}[t]
\centering
\small
\caption{Baseline accuracy (\%) at Iteration~0; o4-mini is reported separately in Section~\ref{sec:rlvr}.}
\label{tab:baseline}
\setlength{\tabcolsep}{3pt}
\begin{tabular}{llccc}
\toprule
Tier & Model & GSM8K & MATH & StrQA \\
\midrule
Fast      & GPT-4o-mini  & 91.2 & 72.8 & 75.5 \\
Mid       & GPT-4.1      & 94.6 & 79.8 & 82.5 \\
          & Cl.~Sonnet~4 & 96.8 & 78.2 & 87.5 \\
Frontier  & GPT-5        & 96.2 & 85.8 & 84.5 \\
          & Cl.~Opus~4.6 & \textbf{97.6} & \textbf{86.0} & \textbf{89.5} \\
Reasoning & o3-mini      & 93.2 & 75.8 & 47.0 \\
\bottomrule
\end{tabular}
\end{table}

\subsection{Accuracy Trajectories}

Table~\ref{tab:gsm8k_iter} reveals five distinct convergence patterns on GSM8K (500 problems each).

\begin{table}[t]
\centering
\small
\caption{Accuracy (\%) across refinement iterations on GSM8K.}
\label{tab:gsm8k_iter}
\setlength{\tabcolsep}{3.5pt}
\begin{tabular}{lcccccc}
\toprule
Model & It.~0 & It.~1 & It.~2 & It.~3 & It.~4 & $\Delta$ \\
\midrule
GPT-4o-mini     & 91.2 & 90.0 & 89.6 & 86.6 & 85.0 & $-6.2$ \\
GPT-4.1         & 94.6 & 94.4 & 94.4 & 94.2 & 94.4 & $-0.2$ \\
Cl.~Sonnet~4    & 96.8 & 96.2 & 96.2 & 95.6 & 95.6 & $-1.2$ \\
GPT-5           & 96.2 & 94.8 & 94.4 & 94.6 & 94.4 & $-1.8$ \\
Cl.~Opus~4.6    & 97.6 & 98.0 & \textbf{98.2} & 98.0 & \textbf{98.2} & $\mathbf{+0.6}$ \\
o3-mini         & 93.2 & 96.2 & 96.6 & 96.6 & 96.6 & $\mathbf{+3.4}$ \\
\bottomrule
\end{tabular}
\end{table}

Five modes emerge:

\textbf{(1) Monotonic degradation} (GPT-4o-mini: 91.2\%$\to$85.0\%, $-6.2$~pp). EIR escalates from 1.3\% to 3.8\% across iterations while ECR remains low, leading to compounding errors.

\textbf{(2) Absorbing lock} (GPT-4.1: $\pm$0.2~pp). After one iteration, both EIR and ECR drop to zero---no answers change. The model reaches a near-fixed point at 94.4\%.

\textbf{(3) Stepwise decline} (Claude Sonnet~4: $-1.2$~pp). Alternating ``slip'' iterations ($-0.6$~pp) and stable iterations produce a gradual decline with dynamic equilibrium.

\textbf{(4) Oscillating near-lock} (GPT-5: $-1.8$~pp). Despite frontier capability (96.2\% baseline), GPT-5 declines $-1.4$~pp at Iteration~1, then oscillates around 94.4\%. Its initial EIR (1.9\%) is the highest among mid/frontier models.

\textbf{(5) Beneficial convergence} (o3-mini: $+3.4$~pp; Opus~4.6: $+0.6$~pp). Both improve and stabilize above baseline. o3-mini maintains \textbf{EIR = 0\%} across all 4 iterations; Opus~4.6 achieves EIR $\approx$ 0.2\% with ECR/EIR $\approx$ 125, far exceeding the equilibrium requirement of 40.7.

The GPT-5 vs.\ Opus~4.6 contrast is instructive: despite similar baseline accuracy (96.2\% vs.\ 97.6\%), a 10$\times$ difference in EIR (1.9\% vs.\ 0.2\%) flips the outcome from degradation to improvement. This demonstrates that the EIR threshold for beneficial self-correction is sharp.

\subsection{EIR/ECR Dynamics}

Table~\ref{tab:eir_ecr} presents per-iteration error transition dynamics for all six models.

\begin{table}[t]
\centering
\scriptsize
\caption{EIR and ECR (\%) on GSM8K (500 problems).}
\label{tab:eir_ecr}
\setlength{\tabcolsep}{3pt}
\begin{tabular}{lcccccccc}
\toprule
 & \multicolumn{2}{c}{0$\to$1} & \multicolumn{2}{c}{1$\to$2} & \multicolumn{2}{c}{2$\to$3} & \multicolumn{2}{c}{3$\to$4} \\
\cmidrule(lr){2-3} \cmidrule(lr){4-5} \cmidrule(lr){6-7} \cmidrule(lr){8-9}
 & EIR & ECR & EIR & ECR & EIR & ECR & EIR & ECR \\
\midrule
GPT-4o-mini & 1.3 & 0.0 & 0.9 & 4.0 & 3.8 & 3.8 & 2.1 & 1.5 \\
GPT-4.1     & 0.4 & 3.7 & 0.0 & 0.0 & 0.2 & 0.0 & 0.0 & 3.4 \\
Cl.~Son.~4  & 0.8 & 6.2 & 0.2 & 5.3 & 0.8 & 5.3 & 0.4 & 9.1 \\
GPT-5       & 1.9 & 10.5 & 0.8 & 7.7 & 0.0 & 3.6 & 0.4 & 3.7 \\
Cl.~Op.~4.6 & 0.2 & \textbf{25.0} & 0.2 & \textbf{20.0} & 0.4 & 11.1 & 0.2 & \textbf{20.0} \\
o3-mini     & \textbf{0.0} & 44.1 & \textbf{0.0} & 10.5 & \textbf{0.0} & 0.0 & \textbf{0.0} & 0.0 \\
\bottomrule
\end{tabular}
\end{table}

Key findings: (1)~o3-mini and Opus~4.6 achieve near-zero EIR, acting as a ``correctness guard''; (2)~beneficial models achieve high ECR (o3-mini: 44.1\%, Opus~4.6: 25.0\% at Iteration~1); (3)~EIR \emph{increases} for weaker models (GPT-4o-mini: 1.3\%$\to$3.8\%), indicating compounding errors; (4)~instruction-tuned models fail the equilibrium condition---for GPT-4o-mini, $\text{ECR}/\text{EIR} \approx 1.0$ vs.\ the required 8.6; (5)~for o3-mini with EIR = 0, $\text{ECR}/\text{EIR} \to \infty$, trivially satisfying Theorem~\ref{thm:equilibrium}. Findings (4)--(5) directly support H2.

\subsection{Baseline Comparisons}

\textbf{Self-Refine.} On 750 GSM8K problems with GPT-4o-mini, Self-Refine's Iteration~0 accuracy (85.7\%) is substantially lower than generic refinement (91.2\%), likely because Self-Refine's prompt format interferes with initial problem-solving. Both methods degrade: $-4.6$~pp for Generic, $-3.6$~pp for Self-Refine. This supports H5: structured self-feedback does not overcome EIR/ECR imbalance.

\textbf{Compute-Equivalent Comparison.} With 3 API calls per problem (Table~\ref{tab:compute}), Self-Consistency~\cite{wang2023selfconsistency} achieves 93.4\%, surpassing baseline by 2.2~pp and outperforming 3-iteration refinement by 6.8~pp.

\begin{table}[t]
\centering
\small
\caption{Compute-equivalent comparison (3~API calls) on GSM8K with GPT-4o-mini. $\Delta$ is relative to the single-shot baseline (91.2\%).}
\label{tab:compute}
\begin{tabular}{lcc}
\toprule
Method & Acc.\ (\%) & $\Delta$ \\
\midrule
Baseline (single shot)    & 91.2 & --- \\
3-iter Generic Refine     & 86.6 & $-4.6$ \\
Self-Refine (3 iters)     & 82.1 & $-9.1$ \\
\textbf{Self-Consistency} & \textbf{93.4} & $\mathbf{+2.2}$ \\
\bottomrule
\end{tabular}
\end{table}

Self-Consistency generates \emph{independent} samples, breaking the sequential dependency chain. With $k=3$ independent samples at $p=0.912$, the theoretical majority-vote accuracy $p^3 + 3p^2(1-p) \approx 97.9\%$; the observed 93.4\% is lower due to correlated errors.

\subsection{Verify-First Prompt Ablation}

Can we \emph{induce} near-zero EIR via prompting? We test a verify-first prompt on GPT-4o-mini: ``Before making any changes, first verify whether your previous answer is correct by re-solving independently. Only change your answer if you find a concrete, specific error.'' We replicate on GPT-4.1 and Claude Sonnet~4.5 (a point refresh of Sonnet~4)---two models whose baseline EIR is already $\lesssim 0.5\%$.

\begin{table}[t]
\centering
\small
\caption{Verify-first ablation on GSM8K (500 problems): accuracy (\%) per iteration; EIR rows (\%) for the degrading case.}
\label{tab:verifyfirst}
\setlength{\tabcolsep}{3pt}
\begin{tabular}{llcccccc}
\toprule
Model & Prompt & It.~0 & It.~1 & It.~2 & It.~3 & It.~4 & $\Delta$ \\
\midrule
GPT-4o-mini & Standard     & 91.2 & 90.0 & 89.6 & 86.6 & 85.0 & $-6.2$ \\
            & Verify-First & 91.2 & 91.2 & 91.4 & 91.4 & \textbf{91.4} & $\mathbf{+0.2}$ \\
GPT-4.1     & Standard     & 94.6 & 94.4 & 94.4 & 94.2 & 94.4 & $-0.2$ \\
            & Verify-First & 94.6 & 95.0 & 95.0 & 94.6 & 94.6 & $0.0$ \\
Sonnet~4.5 & Verify-First & 97.2 & 97.2 & 97.0 & 96.8 & 97.2 & $0.0$ \\
\midrule
GPT-4o-mini & Std.\ EIR    & --- & 1.3 & 0.9 & 3.8 & 2.1 & \\
            & VF EIR       & --- & \textbf{0.0} & \textbf{0.0} & \textbf{0.0} & \textbf{0.0} & \\
\bottomrule
\end{tabular}
\end{table}

On GPT-4o-mini, verify-first achieves \textbf{EIR = 0\%} across all 4 iterations, converting monotonic degradation ($-6.2$~pp) to slight improvement ($+0.2$~pp); the paired McNemar test on iter-4 outcomes gives $p < 10^{-4}$ vs.\ standard refinement, with $\Delta$~Acc $= +6.4$~pp (paired-bootstrap 95\% CI $[+4.2, +8.8]$). On GPT-4.1 and Sonnet~4.5, whose baseline EIR is already below the threshold, verify-first produces no detectable accuracy change (paired McNemar $p = 1.0$ for both)---consistent with the prediction that the intervention is \emph{targeted}: active where baseline EIR is harmful, inert where the threshold is already satisfied. This confirms near-zero EIR is the \emph{sufficient condition} for preventing degradation, achievable through prompt engineering alone. However, it does not achieve the beneficial improvements of o3-mini ($+3.4$~pp), suggesting ECR enhancement requires deeper training-level capabilities.

\subsection{RLVR Validation: o4-mini}
\label{sec:rlvr}

To test generalization, we evaluate o4-mini---a third RLVR model. It exhibits average EIR = 0.2\% with accuracy oscillating within $\pm$0.4~pp of 96.8\%. With three models tested (o3-mini: EIR = 0\%, Opus~4.6: EIR $\approx$ 0.2\%, o4-mini: EIR $\approx$ 0.2\%), the near-zero EIR property appears consistent across all non-degrading models.

\subsection{Response Pattern Analysis}

Table~\ref{tab:response_patterns} reveals a surprising finding: \emph{all} models exhibit high verification phrase ratios (84.8--100\%), including GPT-4o-mini (98.8\%), which degrades the most.

\begin{table}[t]
\centering
\small
\caption{Response patterns on GSM8K (500 $\times$ 4 iterations).}
\label{tab:response_patterns}
\setlength{\tabcolsep}{2.5pt}
\begin{tabular}{lccccc}
\toprule
Model & Ver.\% & Conf.\% & Chg.\% & EIR & Mode \\
\midrule
GPT-4o-mini  & 98.8 & 95.5 & 3.4 & 2.02 & Degrade \\
GPT-4.1      & 99.5 & 99.0 & 0.5 & 0.16 & Absorb \\
Cl.~Son.~4   & 100.0 & 99.1 & 0.9 & 0.57 & Stepwise \\
GPT-5        & 96.2 & 95.2 & 1.2 & 0.78 & Oscillate \\
\midrule
Cl.~Op.~4.6  & 99.9 & 99.2 & 0.7 & 0.26 & Beneficial \\
o3-mini      & 84.8 & 83.8 & 1.2 & 0.00 & Beneficial \\
\bottomrule
\end{tabular}
\end{table}

The distinguishing factor is not \emph{how} models verify but \emph{how often they choose to change answers}---and whether those changes target genuinely incorrect answers. When models change answers on previously correct problems, the change is harmful in nearly 100\% of cases across \emph{all} models. The key mechanism behind near-zero EIR is an \emph{internal ability to accurately assess correctness}, not surface-level verification rhetoric.

\subsection{Statistical Analysis}

Table~\ref{tab:hypothesis} presents hypothesis test results.

\begin{table}[t]
\centering
\small
\caption{Hypothesis tests for GPT-4o-mini on GSM8K ($\alpha = 0.05$).}
\label{tab:hypothesis}
\setlength{\tabcolsep}{3pt}
\begin{tabular}{llccc}
\toprule
Hyp. & Test & Stat. & $p$ & Result \\
\midrule
H1: It.~0$\to$1  & McNemar & 6.00 & 0.014 & Sig.~$\downarrow$ \\
H1: It.~1$\to$2  & McNemar & 0.67 & 0.414 & n.s. \\
H1: It.~2$\to$3  & McNemar & 11.84 & $<$0.001 & Sig.~$\downarrow$ \\
H1: It.~3$\to$4  & McNemar & 6.40 & 0.011 & Sig.~$\downarrow$ \\
H5: Gen.~vs.~SR  & McNemar & 0.50 & 0.480 & n.s. \\
\bottomrule
\end{tabular}
\end{table}

Three of four transitions show significant decreases, providing strong evidence that iterative self-correction is actively harmful for this model. Self-Refine does not significantly outperform generic refinement ($p = 0.48$), confirming H5.

\subsection{Analysis: Why Self-Correction Fails}

\textbf{Pool-size asymmetry.} For GPT-4o-mini on GSM8K, the correctable error pool is small (44/500) while the correct pool susceptible to EIR is large (456/500). Even a small EIR (1.3\%) on the large correct pool produces more errors than a moderate ECR on the small incorrect pool. This instantiates the \emph{accuracy--correction paradox} introduced above: for any model with $\text{Acc} > \pi^*$, self-correction degrades performance.

\textbf{Why o3-mini and Opus~4.6 succeed.} Both break the paradox by achieving near-zero EIR, dramatically reducing pool-size asymmetry. With EIR $\approx$ 0, any nonzero ECR guarantees improvement. We hypothesize both models implement an implicit ``verify-then-edit'' strategy. For o3-mini, this likely emerges from RL training on chain-of-thought reasoning; for Opus~4.6, it may reflect large-scale instruction tuning that develops strong internal consistency checking.

\textbf{Cross-dataset generalization (H3).} GPT-4o-mini's degradation generalizes: on MATH, accuracy declines from 72.8\% to 68.5\% ($-4.3$~pp), confirming the pattern is not dataset-specific.

\textbf{ASC behavior (H4).} ASC halts at Iteration~0 for all 500 GSM8K problems with GPT-4o-mini, correctly identifying that further refinement would hurt. The confidence-elicitation prompt itself, however, costs accuracy (87.4\% vs.\ 91.2\%, $-3.8$~pp), suggesting explicit self-assessment diverts reasoning capacity; a two-stage variant (generate first, refine only uncertain cases) could mitigate this.

\textbf{Two-tier capability model.} Our findings support a two-tier model (Fig.~\ref{fig:twotier}): (1)~\emph{EIR suppression} (achievable via prompt engineering or RL) prevents degradation by stopping the model from changing correct answers; (2)~\emph{ECR enhancement} (likely requiring RL training with verifiable rewards) enables actual improvement by developing the capacity to reliably identify and correct genuine errors. The verify-first prompt achieves tier~1 ($+0.2$~pp) but not tier~2 ($+3.4$~pp for o3-mini), confirming this distinction.

\begin{figure}[t]
\centering
\includegraphics[width=0.85\columnwidth]{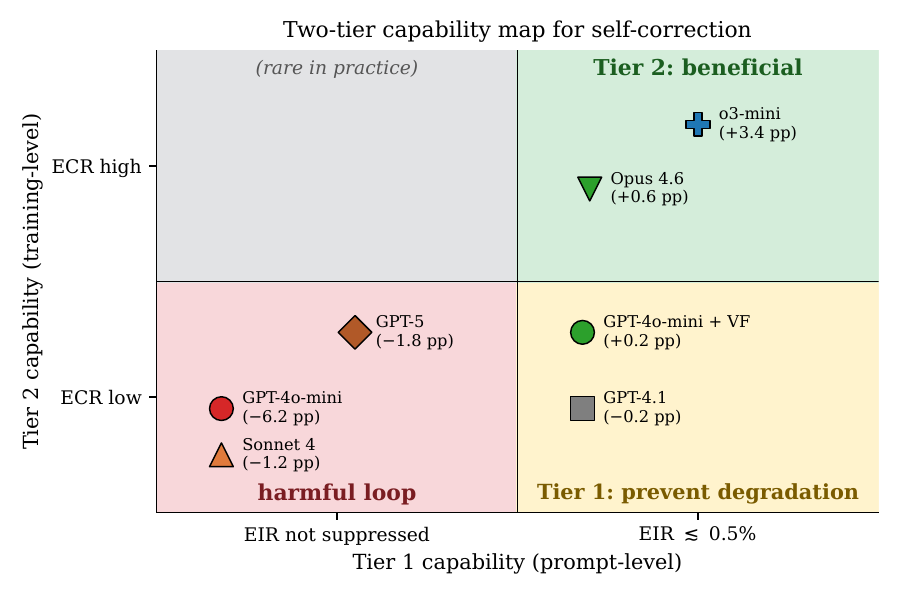}
\caption{Two-tier capability map. Tier~1 (EIR suppression, prompt-level) prevents degradation; Tier~2 (ECR enhancement, training-level) yields actual gains. Verify-First moves GPT-4o-mini out of the harmful quadrant; only RLVR-trained models reach the beneficial quadrant.}
\label{fig:twotier}
\end{figure}

%%%%%%%%%%%%%%%%%%%%%%%%%%%%%%%%%%%%%%%%%%%%%%%%%%%%%%%%%%%%%%%%%%%%%%%%%%%%%%%%
\section{CONCLUSIONS}

This work makes three contributions. \textbf{First}, we cast self-correction as a two-state Markov feedback system and use its standard equilibrium (Theorem~\ref{thm:equilibrium}), steady-state (Theorem~\ref{thm:steadystate}), and convergence (Theorem~\ref{thm:convergence}) properties as an operational diagnostic. For high-accuracy models, $\pi^*$ can be far below baseline, making repeated self-correction harmful.

\textbf{Second}, we reveal a hierarchy of convergence behaviors across seven models. The key predictor is near-zero EIR ($\lesssim$ 0.5\%), not model architecture: o3-mini ($+3.4$~pp, EIR = 0\%), Claude Opus~4.6 ($+0.6$~pp, EIR $\approx$ 0.2\%), and o4-mini ($\pm$0~pp, EIR $\approx$ 0.2\%) all achieve non-degrading self-correction. The GPT-5 vs.\ Opus~4.6 contrast (similar capability, opposite outcomes) shows that a 10$\times$ difference in EIR can flip the closed-loop outcome.

\textbf{Third}, we validate a verify-first prompt that reduces EIR to 0\% through prompt engineering alone, and we use ASC mainly to analyze stopping trade-offs rather than claim a universal performance gain.

Practical implications: (1) for models with $\text{Acc} > \pi^*$, self-correction \emph{reduces} accuracy, so practitioners should estimate EIR on calibration sets; (2) verify-first prompting behaves like lightweight controller design, converting harmful loops into stable ones; (3) Self-Consistency (93.4\%) outperforms iterative refinement (86.6\%) at matched compute.

\subsection{Limitations and Future Directions}

\textbf{Scope.} Our detailed refinement analysis covers seven models on GSM8K with 4 iterations each. Extending to MATH and StrategyQA would further test cross-domain generalization and clarify whether the near-zero EIR threshold remains stable across answer formats and verifier difficulty.

\textbf{Non-stationary rates.} The Markov chain model assumes rates approach stationarity. We observe non-stationarity (EIR increases from 1.3\% to 3.8\%), suggesting time-varying models may be more appropriate. A natural next step is a control-theoretic formulation with iteration-dependent transition dynamics and adaptive stopping thresholds.

\textbf{External feedback.} We evaluate \emph{intrinsic} self-correction (no external oracle). Huang et al.~\cite{huang2024selfcorrect} show that external feedback can make self-correction effective. Incorporating tool outputs, retrieval results, or verifier models could shift the EIR/ECR balance favorably. In control terms, such signals may provide the exogenous correction needed to move the loop from self-referential drift toward error attenuation.

\textbf{Domain generality.} Our evaluation focuses on reasoning with clear correctness criteria; extending to open-ended generation would require redefining EIR/ECR with graded or preference-based notions of correctness. Process-supervised verifiers~\cite{lightman2024verify} suggest one path: a step-graded state in place of $\{C,I\}$ would make harmful self-correction detectable at sub-solution granularity.

\textbf{Compute trade-offs.} The Self-Consistency vs.\ refinement comparison applies at our fixed temperature; recent test-time compute work~\cite{snell2024scaling} indicates the optimal allocation depends on difficulty, motivating a budget-indexed EIR$(B)$/ECR$(B)$ frontier.

\section*{Reproducibility}
We provide anonymous supplementary material with per-iteration trajectories, change logs, EIR/ECR estimates, verify-first outputs, and the ASC reference implementation. All table values and hypothesis tests can be reproduced from the supplied CSVs with a single driver script.

%%%%%%%%%%%%%%%%%%%%%%%%%%%%%%%%%%%%%%%%%%%%%%%%%%%%%%%%%%%%%%%%%%%%%%%%%%%%%%%%

\enlargethispage{\baselineskip}
\balance

\end{document}